\documentclass{article}

\usepackage{microtype}
\usepackage{graphicx}
\usepackage{subfigure}
\usepackage{booktabs} 
\usepackage{multirow}
\usepackage{float}
\usepackage{stfloats}

\usepackage[table]{xcolor}
\usepackage{amsfonts}
\usepackage{amsthm}
\usepackage{graphicx}
\usepackage{duckuments}
\usepackage{comment}
\usepackage{wrapfig}

\usepackage{hyperref}

\usepackage[accepted]{icml2025}

\usepackage{amsmath}
\usepackage{amssymb}
\usepackage{mathtools}
\usepackage{amsthm}

\usepackage[capitalize,noabbrev]{cleveref}

\theoremstyle{plain}
\newtheorem{theorem}{Theorem}[section]

\newtheorem{lemma}[theorem]{Lemma}

\theoremstyle{definition}

\theoremstyle{remark}

\usepackage[textsize=tiny]{todonotes}

\icmltitlerunning{Diffusion-Free Graph Generation with Next-Scale Prediction}

\begin{document}

\twocolumn[
\icmltitle{Diffusion-Free Graph Generation with Next-Scale Prediction}



\icmlsetsymbol{equal}{*}

\begin{icmlauthorlist}
\icmlauthor{Samuel Belkadi}{equal,cam_eng}
\icmlauthor{Steve Hong}{equal,cam_eng}
\icmlauthor{Marian Chen}{equal,cam_eng}
\icmlauthor{Miruna Cretu}{cam_cs}
\icmlauthor{Charles Harris}{cam_cs}
\icmlauthor{Pietro Li\`{o}}{cam_cs}
\end{icmlauthorlist}

\icmlaffiliation{cam_eng}{Department of Engineering, University of Cambridge, UK}
\icmlaffiliation{cam_cs}{Department of Computer Science, University of Cambridge, UK}

\icmlcorrespondingauthor{Samuel Belkadi}{sb2764@cam.ac.uk}
\icmlcorrespondingauthor{Steve Hong}{mdh58@cam.ac.uk}

\icmlkeywords{Multi-scale, Hierarchical, Next-scale, Diffusion-free, Autoregressive, Graph, Molecules, Generation, Transformer, Generative AI, Machine Learning}

\vskip 0.3in
]



\printAffiliationsAndNotice{\icmlEqualContribution} 

\begin{abstract}
Autoregressive models excel in efficiency and plug directly into the transformer ecosystem, delivering robust generalization, predictable scalability, and seamless workflows such as fine-tuning and parallelized training. However, they require an explicit sequence order, which contradicts the unordered nature of graphs. In contrast, diffusion models maintain permutation invariance and enable one-shot generation but require up to thousands of denoising steps and additional features for expressivity, leading to high computational costs. Inspired by recent breakthroughs in image generation, especially the success of visual autoregressive methods, we propose \texttt{MAG}, a novel diffusion-free graph generation framework based on next-scale prediction. By leveraging a hierarchy of latent representations, the model progressively generates scales of the entire graph without the need for explicit node ordering. Experiments on both generic and molecular graph datasets demonstrated the potential of this method, achieving inference speedups of up to three orders of magnitude over state-of-the-art methods, while preserving high-quality generation.
\end{abstract}

\section{Introduction}
\label{submission}
Graphs provide a natural and flexible information representation across a wide range of domains, including social networks, biological and molecular structures, recommender systems, and infrastructural networks. Consequently, the ability to learn a graph distribution from data and generate realistic graphs is pivotal for applications such as network science, drug discovery, and protein design.

Despite significant progress in generative models for language and images, graph generation remains challenging due to its inherent combinatorial nature. Specifically, graphs are naturally \textbf{high-dimensional} and \textbf{discrete} with \textbf{varying sizes}, contrasting with the continuous and fixed-size techniques that cannot be directly applied to them. Furthermore, rich substructures in graphs necessitate an expressive model capable of capturing higher-order motifs and interactions.

Several graph generative models have been proposed to address some of these challenges, undertaking approaches based on autoregressive (AR) models \citep{you2018graphrnn}, varational autoencoders (VAE) \citep{jin2018junction}, generative adversarial networks \citep{de2018molgan}, and diffusion models \citep{niu2020permutation, jo2022GDSS, vignac2023digress}. Among these, diffusion and autoregressive models excel in performance and offer complementary strengths—diffusion’s permutation invariance and AR’s efficiency—that make them especially promising avenues  \cite{kong2023autoregressive, zhao2024pard, yan2024swingnn}.

On one hand, diffusion models offer the ability to achieve exchangeable probability in combination with permutation equivariant networks, under certain conditions \citep{niu2020permutation, jo2022GDSS}. However, given the high-dimensional nature of graphs and their complex internal dependencies, directly modeling the joint distribution of all nodes and edges presents significant challenges. Recent work by \citet{zhao2024pard} demonstrates that, as a cost for permutation invariance, capturing the full joint distribution and solving the transformation difficulty via diffusion requires thousands of sampling and denoising steps as well as additional node-, edge-, and graph-level features, such as eigenvectors, to break symmetries and achieve high generation quality, rendering diffusion a promising but very expensive approach. 

On the other hand, in addition to their improved efficiency, studies into the success of AR models have highlighted their \textbf{scalability} and \textbf{generalizabilty}. The former, as exemplified by scaling laws \citep{henighan2020scaling}, allows the prediction of large models' performance to better guide resource allocation during training, while the latter, as evidenced by zero-shot and few-shot learning \citep{sanh2021multitask}, underscores the adaptability of unsupervised-trained models to unseen tasks. Furthermore, the aforementioned discrete and varying-size properties of graphs suggest a promising fit for autoregressive modeling. However, unlike natural language sentences with an inherent left-to-right ordering, the order of nodes in a graph must be explicitly defined for unidirectional autoregressive learning. Previous AR methods \citep{you2018graphrnn, liao2020efficient, shi2020graphaf} attempted to enforce node ordering or approximate the marginalization over permutations such that, once flattened, one can train an autoregressive model to maximize its likelihood via next-token prediction. This gives rise to the three main issues of autoregressive graph generation:

\begin{enumerate}
	\item \textbf{Long-range dependency.} Generating nodes one at a time results in sequential dependencies that make it difficult to capture long-range interactions and global structures. Furthermore, early predictions errors can propagate throughout the generation process, and dependencies may decay when working on larger graphs.
	\item \textbf{Forced node ordering.} Enforcing a specific node ordering is at odds with the inherent invariance property of graphs, forcing the model to learn over a specific ordering, thereby increasing complexity and reducing generalization.
    \item \textbf{Inefficiency.} Generating a graph sequence $(x_1, x_2, ..., x_{N})$ with a conventional self-attention transformer incurs $\mathcal{O}(n)$ autoregressive steps and $\mathcal{O}(n^3)$ computational cost (see Appendix \ref{complexity proof}).
\end{enumerate}

Drawing inspiration from recent breakthroughs in image generation—notably the Visual Autoregressive Modeling (VAR) framework \cite{tian2024visual, ren2025xar}—our proposed method suggests a new ordering for graphs that is \textbf{multi-scale, coarse-to-fine}. We introduce \texttt{MAG}, a novel transformer-based graph generation framework based on next-scale prediction instead of traditional \textit{next-node} and \textit{next-edge} approaches, thereby infusing autoregressive properties with graphs' structural constraints. As shown in Figure \ref{fig:architecture_overview}, our method first encodes a graph into multiple scales of latent representations and learns to generate scales autoregressively, starting from a singular token and progressively expanding its resolution. At each step, a transformer predicts the next-resolution token map, conditioned on previous scales and a class label $\texttt{[C]}$ for conditional synthesis. 

\texttt{MAG} poses significant improvements over traditional autoregressive methods on both generic and molecular graph datasets. Notably, it reduces the performance gap with diffusion baselines, surpassing them in wall-clock training and inference time, with strong performance. This provides strong evidence regarding the effectiveness of next-scale prediction for high-quality graph generation. In summary, our contributions are threefold:

\begin{figure*}[t]
	\centering
	\includegraphics[width=0.86\linewidth]{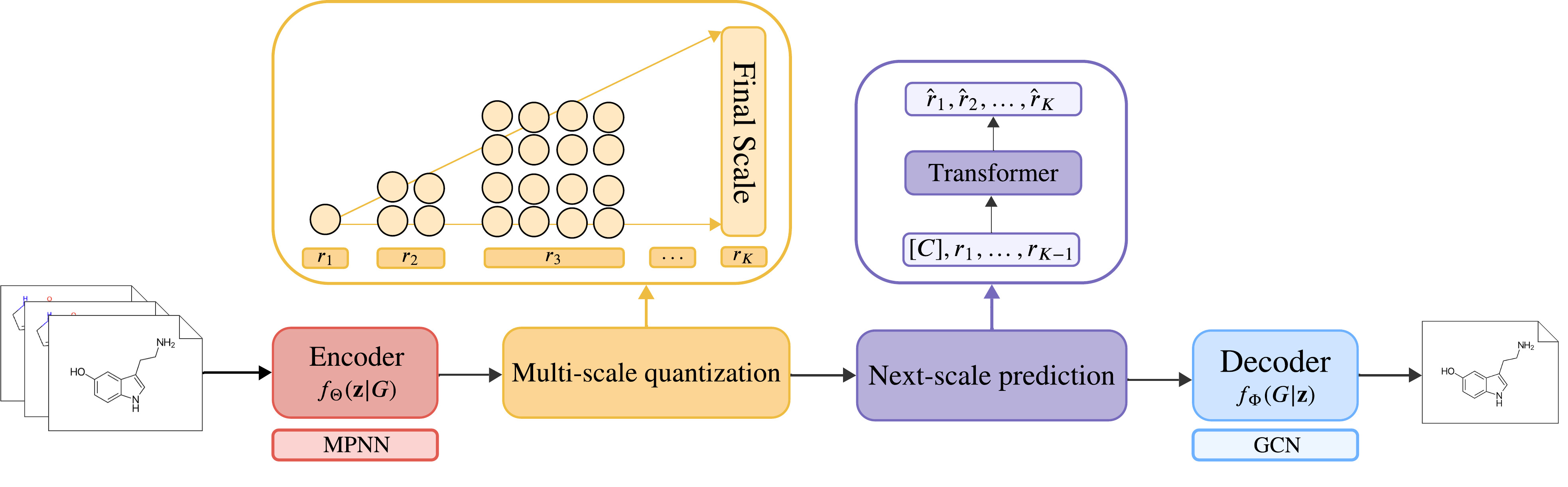}
    \caption{\textbf{Overview of the Mutli-Scale Graph Generation Framework.} The input graph is first encoded through an MPNN, which produces a latent representation that is quantized at multiple scales. A Transformer then predicts the scales autoregressively via a coarse-to-fine next-scale process, and a GCN-based decoder reconstructs the graph from the final discrete latent representation $r_K$.
}
	\label{fig:architecture_overview}
\end{figure*}

\begin{enumerate}
    \item A novel next-scale prediction framework for hierarchical graph generation that overcomes the limitations of traditional AR approaches on graph properties.
    \item A thorough theoretical and empirical analysis of how multi-scale latent representations can effectively capture high-order motifs while remaining significantly more efficient than counterparts.
    \item Experiments on diverse graph datasets, demonstrating significant improvements over AR baselines and promising performance when compared to state-of-the-art diffusion methods.
\end{enumerate}

\section{Related Work}

\paragraph{Diffusion and score-based models} 
Diffusion models achieve permutation-invariance by parameterizing the denoising/score function with a permutation equivariant network, at the cost of up to thousands of sampling steps. EDP-GNN \cite{niu2020permutation} is the first work to adapt score-based models to graph generation by representing graphs as matrices with continuous values. GDSS \cite{jo2022GDSS} generalizes this approach by leveraging SDE-based diffusion and replacing adjacency matrices with node and edge features. However, by using continuous-state diffusion, these approaches ignore graphs' discrete nature, resulting in fully connected graphs. DiGress \cite{vignac2023digress} is the first to apply discrete-state diffusion, achieving significant improvements over continuous methods. Yet, they require additional structural and domain-specific features to break symmetries and achieve high generation quality, further increasing computational costs. In response, SwinGNN \cite{yan2024swingnn} argues that learning exchangeable probabilities with equivariant architectures is challenging, instead proposing fixed node orderings and non-equivariant score functions to alleviate complexity while maintaining competitive performance. 

\paragraph{Autoregressive models} 
Traditional AR models construct graphs by adding nodes and edges sequentially. Although this approach aligns with the discrete nature of graphs, it faces a fundamental challenge that there is no inherent order in graph generation. Various strategies have been explored to address this by enforcing some node ordering and approximating the marginalization over permutations. \citet{li2018learning} suggest using random or deterministic empirical orderings. GraphRNN \cite{you2018graphrnn} aligns permutations with Breadth-First Search (BFS) ordering through a many-to-one mapping. GRANs \cite{liao2020efficient} propose an approach to marginalize over a family of canonical node orderings, such as node degree descending, BFS trees rooted at the highest degree node, and k-core orderings. \citet{chen2021order} eliminate ad-hoc orderings altogether by modeling the conditional probability of orderings with another AR model to estimate the marginalized probabilities for both the generative and ordering-selection models. Despite these efforts, all such methods remain fundamentally limited by the need to impose artificial node orderings.

\paragraph{Hybrid and hierarchical models}
Several recent works seek to blend the strengths of diffusion and autoregressive (AR) paradigms.  PARD \cite{zhao2024pard} first partitions the graph into blocks, applies diffusion within each block and AR between them, and enforces permutation invariance via \emph{partial ordering}.  While effective, its block-wise diffusion introduces extra implementation complexity and still incurs substantial sampling costs \cite{hou2024improving}.  GraphArm \cite{kong2023autoregressive} instead embeds diffusion steps into an AR framework—at each forward pass, a single node and its edges stochastically decay—making the model both permutation- and order-sensitive.  Beyond these hybrids, pure coarse-to-fine strategies have also been explored: Bergmeister et al.\ \cite{bergmeister2024efficient} grow a full graph from one seed node using localized denoising diffusion and spectral cues, matching diffusion-only runtimes but retaining their high sampling overhead.  In the molecular domain, JT-VAE \cite{jin2018junction} and HierVAE \cite{jin2020hierarchical} build molecules by assembling chemically valid substructures, and HiGGs \cite{davies2023size} hierarchically samples entire graphs across multiple resolutions.  However, all of these rely on either block-level diffusion, node-level decay processes, or strong domain priors (e.g., motif libraries or spectral conditioning). In contrast, our framework is fully domain-agnostic and permutation-equivariant, leveraging multi-scale quantized latent maps with a generic transformer—without any hand-crafted substructure priors.

\section{Method}

Consider a graph defined by the quadruple 
\(G = (\mathcal{V}, E, \mathbf{X}, \mathbf{B})\), where \(\mathcal{V}\) denotes the set of $N$ vertices, \(E \subseteq \mathcal{V} \times \mathcal{V}\) represents the set of edges, \(\mathbf{X} \in \mathbb{R}^{N \times D}\) is a matrix of $D$-dimensional node features, and \(\mathbf{B} \in \mathbb{R}^{N \times N \times F}\) comprises edge attributes with dimensions \(F\). Given a collection of \(M\) observed graphs, \(\mathcal{G} = \{G_i\}_{i=1}^{M}\), the task of graph generation is to model the underlying distribution \(p(\mathcal{G})\) in order to generate new graphs, \(G_{\text{sample}} \sim p(\mathcal{G})\).

\subsection{Preliminary: Autoregression by Scale}

AR models represent the joint distribution over \(N\) random variables by employing the chain rule of probability. In particular, they decompose the generation process into sequential steps, where each step determines the subsequent action based on the current subgraph. Traditionally, previous works have taken a \emph{next-token prediction} approach, generating the graph one node and its edges at a time. Specifically, at step \(i\), we introduce node \(G_i^\pi\) along with any new edges connecting it to nodes in \(\{G_1^\pi, G_2^\pi, \ldots, G_{i-1}^\pi\}\). The general formulation of AR models for graphs is given by:
\[
p(\mathcal{G}^{\pi}) = \prod_{i=1}^{N} p\bigl(G_i^{\pi} \mid G_1^{\pi}, G_2^{\pi}, \dots,G_{i-1}^{\pi}).
\]
Since AR models proceed in a unidirectional sequential manner, applying them requires a predetermined ordering \(\pi\) of nodes in the graph.

Alternatively, rather than employing \textit{next-token prediction}, we propose utilizing \emph{next-scale prediction}, wherein the autoregressive unit is a latent representation of the entire graph rather than a single token. To implement this, we first obtain latent representations at multiple scales and train a transformer to generate these latent maps. Similarly to \citet{tian2024visual}, this leads to a two-stage architecture: \textbf{(1)} a VQ-VAE tokenizer that quantizes the input graphs into latent maps at multiple scales (refer to Section \ref{permutation}), and \textbf{(2)} a transformer that predicts next-scale latent maps (refer to Section \ref{AR generation}).

\begin{figure*}[t]
    \centering    
    \includegraphics[width=0.90\linewidth]{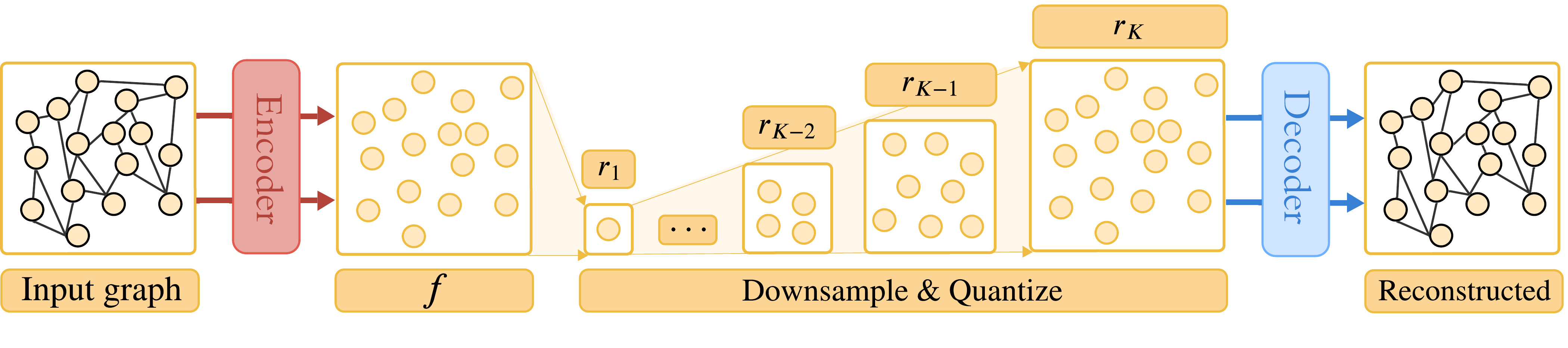}
    \caption{%
    \textbf{Permutation Equivariant Multi-Scale VQ-VAE.}
    The encoder converts the input graph into a continuous latent representation \(f\), which is downsampled and quantized into discrete token maps \(\{r_1, \dots, r_K\}\). The final token map \(r_K\) is then passed to the GCN-based decoder to reconstruct the original graph, preserving permutation equivariance throughout.}
    \label{fig:vqvae}
\end{figure*}

\subsection{Permutation Equivariant Multi-Scale Tokenizer} \label{permutation}

We propose a permutation-equivariant quantized autoencoder to encode a graph $G$ to $K$ multi-scale discrete latent maps $R = \{r_1, r_2, ..., r_K\}$, where each scale $r_k$ depends only on its prefix $\{r_1, r_2, ..., r_{k-1}\}$. For quantization, a shared codebook $Z$ is used across all scales, ensuring that each token in $r_k$ belongs to the same vocabulary. An overview of the tokenizer is given in Figure \ref{fig:vqvae}.

\paragraph{Encoding} 
Given an input graph $G \in \mathbb{R}^{N \times D} \times \mathbb{R}^{N \times N \times F}$, an autoencoder $\mathcal{E}(\cdot) : \mathbb{R}^{N \times D} \times \mathbb{R}^{N \times N \times F} \rightarrow \mathbb{R}^{N \times C}$ is used to convert $G$ into a continuous latent representation $f$:
\[
f = \mathcal{E}(G), \quad f \in \mathbb{R}^{N \times C},
\]
where $C$ is the latent's feature dimension defined empirically. In order to encode the graph's node and edge features into a unified latent space $\mathbb{R}^{N \times C}$, we leverage a series of standard MPNNs, such that edge features are propagated within node features before compressing node dimensions to $C$ and dropping remaining edges. 

Concretely, each MPNN layer performs message passing by first computing learned linear transformations of neighbor node embeddings and their connecting edge attributes, aggregating these via a sum operation, and then applying a ReLU nonlinearity. We stack $L_e$ such layers—with residual connections and layer normalization—to ensure stable training and encode higher-order structural patterns throughout the graph within nodes.

$f$ is then downsampled to create $K$ different coarse-to-fine latent scales $F = \{f_1, f_2, ..., f_K\}$, with $f_K = f$. Intermediate scales are learned by the model to autoregressively generate the graph's latent representation, scale-wise. 

\paragraph{Quantizing} 
A quantizer is used to convert each scale's latent representation into discrete tokens. The quantizer $\mathcal{Q}(\cdot) : \mathbb{R}^{N\times C} \rightarrow \mathbb{R}^{N \times C_Z}$ includes a codebook $Z \in \mathbb{R}^{V \times C_Z}$ containing $V$ learnable embeddings, for which each feature's dimension is maintained by convention, i.e., $C_Z = C$. During quantization, $q=\mathcal{Q}(f)$ is obtained by replacing each feature vector $f^{(k, i)}$ from the multi-scale latent features $F$ by its nearest code $q^{(k,i)}$ in Euclidean distance:
\[
    q^{(k,i)} = ( \text{argmin}_{v \in V} || \text{Select}(Z, v) - f^{(k,i)}||_2 ) \in V,
\]
where $\text{Select}(Z, v)$ denotes selecting the $v^{th}$ vector in codebook $Z$, and $ i \leq N$. Specifically, $\mathcal{Q}(\cdot)$ is trained by fetching all $q^{(k,i)}$ and minimizing the distance between $q$ and $f$.

\paragraph{Decoding} 
Finally, a decoder is used to reconstruct the original graph given its final discrete latent representation, $r_K$. The decoder $\mathcal{D}(\cdot) : \mathbb{R}^{N \times C_Z} \rightarrow \mathbb{R}^{N \times D} \times \mathbb{R}^{N \times N \times F}$ recovers latent representations for edges between all pairs of nodes, and maps node and edge features back to their original dimensions $\mathbb{R}^{N \times D}$ and $\mathbb{R}^{N \times N \times F}$, respectively. 

We implement $\mathcal{D}(\cdot)$ with $L_d$ standard graph convolutional layers, interleaved with ReLU activations and layer normalization to stabilize training and capture higher-order structures. After the final GCN layer, edge logits are recovered by passing all concatenated pairs of node embeddings through an MLP layer, and node features are mapped back to input dimensions, $\mathbb{R}^D$, via a linear projection.

Because this architecture preserves node ordering and feature cardinality throughout encoding, quantization, and decoding, the proposed multi-scale autoencoder remains \textbf{permutation equivariant}. 
Note that any off-the-shelf autoencoder can be used, since downsampled scales are obtained by interpolating, and only the embedding $f$ (and thus final scale $r_K$) is used for decoding. This effectively makes the scales independent of the rest of the architecture.

\begin{figure*}
	\centering
\includegraphics[width=0.92\linewidth]{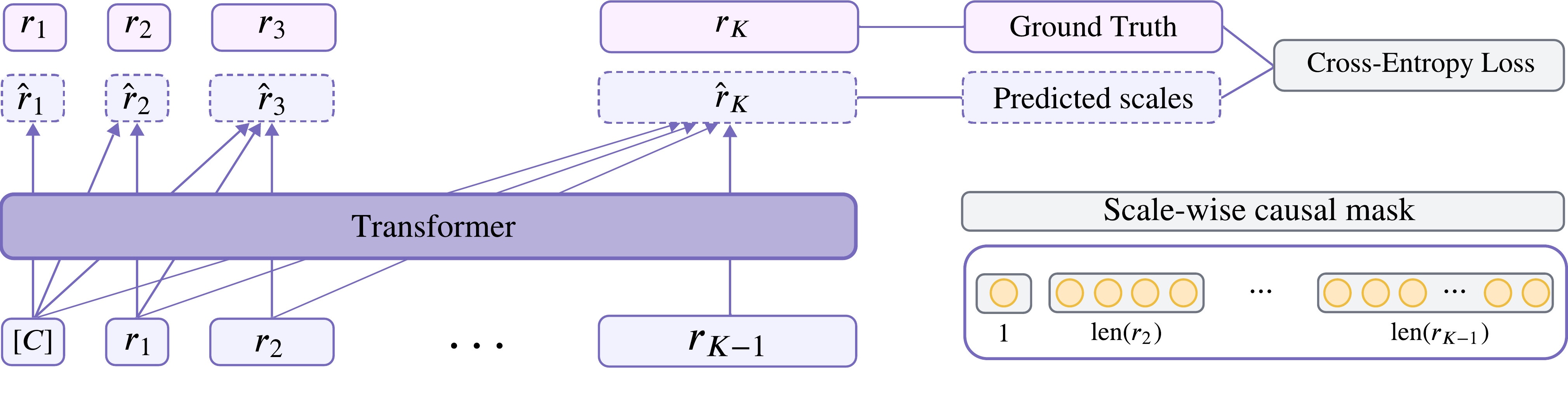}
    \caption{%
    \textbf{Autoregressive Next-scale Prediction with Scale-Wise Causal Mask.}
    At each step \(k\), the transformer attends to all previously generated scales \(\{r_1, \dots, r_{k-1}\}\) and predicts the next-scale token map \(r_k\) all at once. A scale-wise causal mask ensures that each scale only depends on its prefix. Cross-entropy loss is used for training the transformer.%
    }
    \label{fig:ar_generation}
\end{figure*}

\subsection{Autoregressive Graph Generation Through Next-Scale Prediction} \label{AR generation}

We define the autoregressive modeling on graphs by shifting from \textit{next-node} and \textit{next-edge} predictions to \textit{next-scale} predictions. Here, the autoregressive unit is \textit{an entire token map}, rather than \textit{a single token}. Given the quantized scales $\{r_1, \ldots, r_K\}$, the autoregressive likelihood is expressed as:
\[
p(r_1, r_2, \ldots, r_K) = \prod_{k=1}^{K} p(r_k \mid r_1, r_2, \ldots, r_{k-1}),
\]
where each autoregressive unit $r_k \in [V]^{n_k}$ is the token map at scale $k$ containing $n_k$ tokens, and the sequence $(r_1, r_2, \ldots, r_{k-1})$ serves as the prefix for $r_k$. In the $k$-th autoregressive step, we predict the $k$-th scale, such that all $n_k$ tokens in $r_k$ are generated in parallel, conditioned on $r_k$’s prefix. Note that during training, a block-wise causal attention mask is used to ensure that each $r_k$ can only attend to its prefix $r_{< k}$. During inference, kv-caching is used and no mask is needed. An overview of the generative model's  architecture is displayed in Figure \ref{fig:ar_generation}.

\paragraph{Level embedding}
A level embedding is learned to encode the scale at which a token operates within the input structure. This is implemented similarily to GPT's segment embedding \cite{devlin2019bert}, such that learned signals enhance the model’s capacity to capture multi-scale contextual dependencies. However, it is important to note that no additional positional encoding is used within scales, in order to leverage Attention’s permutation-invariance property.

\paragraph{Scale sampling}
During training and inference, we select scale sizes from a predefined set of lengths. Specifically, we fix the first scale to $1$ and the final scale to $N$, where $N$ is the target graph size. Intermediate scales are chosen from the predefined set (e.g., $R = \{1, 2, 4, 6, 9\}$), such that $\forall r_k \in R, r_k \leq N$; and the final scale $r_K = N$ (e.g., truncating to $\{1, 2, 4, 5\}$ for a graph of $5$ nodes). Typically, we set the number of scales to be logarithmic in the number of nodes. Finally, to handle varying graph sizes and enable parallel generation, padding is used with appropriate masking. 

\paragraph{Implementation details}
We adopt a standard decoder-only Transformer architecture with adaptive layer normalization (AdaLN), and input tokens are embedded via a linear layer. Unlike Graphormer \cite{ying2021do} or DiGress \cite{vignac2023digress}, we do not use any additional positional embeddings.
For class-conditional synthesis, a class embedding serves both as the start token $\texttt{[C]}$ and the conditioning input for AdaLN. Then, at each autoregressive step, the next scale is predicted such that all tokens in that scale are generated in parallel through a series of multi-head self-attention blocks. Each output embedding is passed through a classifier head to predict a categorical distribution over the quantization codebook, from which the most probable codeword is selected using a \texttt{softmax} operation.
Dropout and layer normalization are employed throughout the architecture to aid convergence. Additionally, query and key vectors are normalized to unit vectors prior to computing attention scores to improve training stability.
To support reproducibility, the hyperparameters used for training and inference are listed in Appendix~\ref{hyperparameters}.

\paragraph{Sampling at inference}
At inference, we apply top-$k$ and top-$p$ (nucleus) sampling to introduce diversity and control over the generation process. These methods are applied to the predicted categorical distributions at each decoding step.

\begin{figure*}[!hbt]
  \centering
  \includegraphics[width=\linewidth]{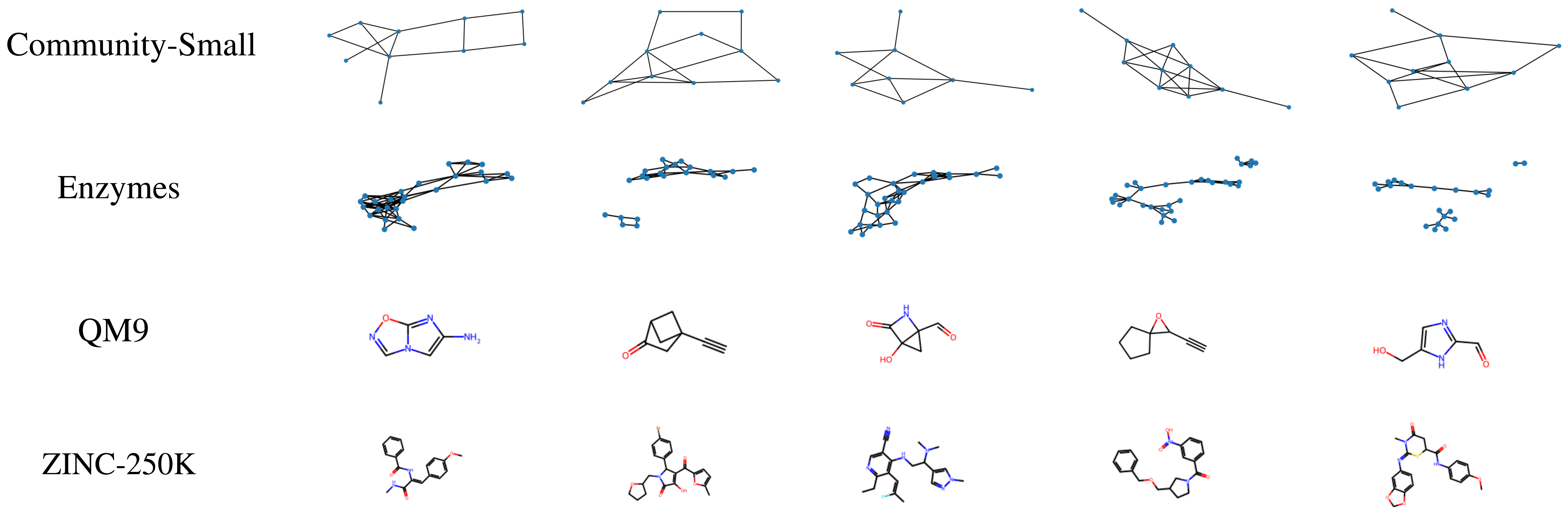}
  \caption{Samples of generated generic and molecular graphs by \texttt{MAG}.}
  \label{fig:samples}
\end{figure*}

\paragraph{Discussion} 

As a result, multi-scale prediction offers substantial properties that allow \texttt{MAG} to address the three previously mentioned issues as follows:

\begin{enumerate}
    \item Given the limited number of intermediate scales in the generations, long-range dependency is reduced from $\mathcal{O}(n)$ to $\mathcal{O}(k)$, where \(k \approx \log(n)\) is the number of scales. Furthermore, sequential dependency is strengthened by generating a scale of the whole graph at each autoregressive step.
    \item No explicit node ordering is needed as \textbf{(i)} the tokenizer is permutation equivariant, and \textbf{(ii)} tokens in each map, $r_k$, are fully correlated and generated in parallel. 
    \item By generating an entire token map in parallel at each scale, the generation's complexity is reduced from \(\mathcal{O}(n^3)\) to \(\mathcal{O}(n^2)\). The proof is detailed in Appendix \ref{complexity proof}. 
\end{enumerate}

\section{Experiments}

We empirically evaluate our model on experiments covering both generic and molecular graph generation. These two domains are chosen to benchmark the model's versatility: generic graphs allow the assessment of structural fidelity and scalability, while molecular graphs pose domain-specific challenges like chemical validity and functional constraints. This dual setting probes \texttt{MAG}'s ability to generate diverse graphs under different structural and semantic constraints.

Generation quality is measured using a combination of distributional and domain-specific metrics. Moreover, we evaluate generation wall-clock time against state-of-the-art diffusion, hybrid, and VAE models to compare their inference efficiency. Both training and evaluation procedures were executed in five independent trials, with error bars denoting the robustness of the experiments.

\paragraph{Baselines}
We compare \texttt{MAG} to state-of-the-art models trained for both generic and molecule generation. Specifically, baselines include diffusion models, such as EDP-GNN \citep{niu2020permutation}, GDSS \citep{jo2022GDSS}, and DiGress \citep{vignac2023digress}; hybrid models, such as GraphAF \citep{shi2020graphaf}, GraphDF \citep{luo2021graphdf} and GraphRNN \cite{you2018graphrnn} (generic graphs only); and VAE models, such as GraphVAE \citep{simonovsky2018graphvae} and GramVAE \citep{kusner2017grammar} (molecular graphs only).

\subsection{Generic Graph Generation}

\paragraph{Experimental Setup}
We evaluate \texttt{MAG} at generic graph generation on the Community-small \cite{jo2022GDSS} and Enzymes \cite{Schomburg2004BRENDA} datasets. Community-small is a synthetic dataset containing 100 random graphs consisting of two equal-sized communities, with between 12 and 20 nodes per graph. Here, the node and edge representations are binary (i.e., absence or presence). The Enzymes dataset contains 587 enzyme graphs with 10 to 125 nodes per sample. Each of the graphs represents a protein as a set of nodes—each node carrying a 3-dimensional one-hot feature for $\alpha$-helix, $\beta$-strand or coil—and binary edges indicating sequence adjacency or the three nearest spatial neighbors. There are no edge attributes, and each graph is annotated with one of six top-level EC enzyme class labels.

These datasets provide a wide range of graph properties, from small structured communities in Community-small to larger, more complex protein graphs in Enzymes. We aim to investigate the ability of multi-scale graph generation to generate graphs of different sizes and complexities, providing insights into its scalability and generalization abilities.

To ensure fair comparison, we follow the same setup as \citep{kong2023autoregressive, yan2024swingnn, jo2022GDSS}, randomly splitting the dataset into 80\% for training and 20\% for testing. Performance is evaluated by comparing the structure of the generated graphs with those from the dataset according to the maximum mean discrepancy (MMD) of statistics including node degrees, clustering coefficients, and orbit counts. Moreover, the inference speed of \texttt{MAG} is compared to the baselines at generating graphs across varying sizes. Samples of generated graphs are given in Figure \ref{fig:samples}.

\begin{table*}[!htbp]
\centering
\renewcommand{\arraystretch}{1.3}
\caption{Results on \textbf{Community-small} and \textbf{Enzymes} benchmark datasets. Wall-clock inference times are reported in seconds for 50 and 100 samples, respectively, on an NVIDIA GeForce RTX 4060 GPU.}
\vspace{4.5pt}
\label{tab:no_mol_results}
\begin{tabular}{lcccccccc}
\toprule
\multirow{2}{*}{\textbf{Models}} 
& \multicolumn{4}{c}{\textbf{Community-S} ($|\mathcal{V}| \in [12,20]$, 100 obs.)} 
& \multicolumn{4}{c}{\textbf{Enzymes} ($|\mathcal{V}| \in [10,125]$, 587 obs.)} \\
\cmidrule(lr){2-5} \cmidrule(lr){6-9}
& Deg. $\downarrow$ & Clus. $\downarrow$ & Orbit. $\downarrow$ & Time (s) $\downarrow$
& Deg. $\downarrow$ & Clus. $\downarrow$ & Orbit. $\downarrow$ & Time (s) $\downarrow$ \\
\midrule
\rowcolor{gray!10} 
\multicolumn{9}{c}{\textit{Diffusion Models}}\\
EDP-GNN   & 0.053 & 0.144 & 0.026 & 1.94$e^4$ & 0.023 & 0.268 & 0.082 & 2.61$e^4$ \\
GDSS      & \textbf{0.045} & 0.086 & 0.007 & 2.87$e^3$ & 0.026 & 0.061 & 0.009 & 1.74$e^3$ \\
DiGress   & 0.047 & \textbf{0.041} & 0.026 & 38.1 & 0.004 & 0.083 & 0.002 & 58.8 \\
\midrule
\rowcolor{gray!10} 
\multicolumn{9}{c}{\textit{Hybrid Models}}\\
GraphAF   & 0.180 & 0.200 & 0.020 & 1.29$e^2$ & 1.669 & 1.283 & 0.266 & 46.1 \\
GraphDF   & 0.060 & 0.120 & 0.030 & 68.8 & 1.503 & 1.061 & 0.202 & 68.1 \\
GraphRNN  & 0.080 & 0.120 & 0.040 & 1.10$e^2$ & 0.017 & 0.062 & 0.046 & 1.25$e^2$ \\
\midrule
\rowcolor{gray!10} 
\multicolumn{9}{c}{\textit{Varational Autoencoders}}\\
GraphVAE  & 0.350 & 0.980 & 0.540 & 1.34 & 1.369 & 0.629 & 0.191 & 2.00 \\
\midrule
\textbf{Ours}  & 0.077 & 0.102 & 0.022 & \textbf{0.187} & 1.753  & 1.148 & 0.223 & \textbf{1.88} \\
\textbf{$\pm$ std.}  
& $\pm$0.007 & $\pm$0.009 & $\pm$0.003 & $\pm$0.015  
& $\pm$0.150 & $\pm$0.100 & $\pm$0.034 & $\pm$0.200 \\
\bottomrule
\end{tabular}
\end{table*}

\paragraph{Results}

\texttt{MAG} achieves strong performance with inference times orders of magnitude faster (Table \ref{tab:no_mol_results}). On the smaller Community-small dataset, it consistently outperforms both hybrid models and VAEs. Notably, its performance is the closest to state-of-the-art GDSS and DiGress models, while being $204\times$ and $15{,}347\times$ faster, respectively.

On the more complex Enzymes dataset, \texttt{MAG} exhibits limitations in qualitative performance. Similarly to hybrid models, it struggles with generating larger enzyme graphs, suggesting possible scalability challenges. This may be due to the increased number of scales required for larger graphs, resulting in higher sensitivity to predefined scale sizes. Despite this limitation, \texttt{MAG}'s computational efficiency scales robustly with graph size, indicating strong potential should its performance be further improved.

In terms of wall-clock inference time, \texttt{MAG} delivers significant improvements across all datasets, surpassing all baselines by several orders of magnitude. Remarkably, it achieves inference speeds up to \textbf{thousands of times faster} than state-of-the-art models.

\subsection{Molecule Generation} 

\paragraph{Experimental Setup}
\texttt{MAG} is further extended to generate molecular graphs with node and edge features, and evaluated on two popular molecular datasets, namely QM9 \cite{ramakrishnan2014qm9} and ZINC-250K \cite{zinc_subset}. The former contains 133,885 small organic molecules of up to nine heavy atoms with types C, O, N and F. The latter is a subset of the ZINC database \cite{irwin2012zinc} and consists of 250,000 molecules with up to 38 heavy atoms of nine types.

In contrast to generic graph generation, molecules present a unique challenge for multi-scale graph generation due to their strict structural constraints and varying complexity. Evaluating \texttt{MAG} on both molecular datasets assesses its ability to generate realistic molecules ranging from small organic compounds to larger chemical structures, while adhering to chemical constraints like valency, introducing additional challenges for the model to maintain validity.

The quality of generated molecules is measured in terms of validity, uniqueness, novelty, and Fréchet ChemNet Distance (FCD). Details on these metrics are delineated in Appendix \ref{app:metrics}. Training and evaluation were each performed in five independent replicates, with the resulting performance robustness illustrated by error bars. In addition, the wall-clock inference time to generate 1,000 molecules from both datasets is recorded and compared against baselines. Samples of generated molecules are displayed in Figure \ref{fig:samples}. 


\paragraph{Results} Performance on molecular datasets are given in Table \ref{tab:mol_results}. \texttt{MAG} consistently outperforms VAE baselines across most metrics. Moreover, \texttt{MAG} achieves performance comparable to that of recent hybrid models. Compared to state-of-the-art diffusion models, \texttt{MAG} achieves the highest uniqueness and matches most of them in novelty. Regarding validity, although it trails slightly behind GDSS, it proves to outperform EDP-GNN. Of all baselines, only DiGress exhibits a notable performance advantage over \texttt{MAG}; however, DiGress explicitly incorporates molecular-specific properties, while \texttt{MAG} approaches the task in a raw, domain-agnostic manner, which may account for this performance gap. For the larger and more complex ZINC-250k dataset, \texttt{MAG} outperforms VAE baselines across all metrics and delivers results comparable to those of hybrid models, while remaining slightly behind diffusion-based methods. These findings are consistent with observations on smaller molecules; however, \texttt{MAG} demonstrates greater consistency across different molecule sizes compared to diffusion models. Additionally, we observe reduced overhead when scaling to larger molecules relative to the results on the Enzymes dataset. 
Regarding computational efficiency, \texttt{MAG} further delivers substantial improvements, achieving inference speeds up to \textbf{hundreds of times faster} than the state of the art.

\begin{table*}
\centering
\renewcommand{\arraystretch}{1.3}
\caption{Results on \textbf{QM9} and \textbf{ZINC250k} molecule datasets. Wall-clock inference times are reported in seconds for 1,000 samples, respectively, on an NVIDIA GeForce RTX 4060 GPU (Val: Validity; Uni: Uniqueness; Nov: Novelty; FCD: Fréchet ChemNet Distance). Entries marked “--” denote results that were not reported or for which reproduction was not possible.}
\vspace{4.5pt}
\label{tab:mol_results}
\begin{tabular}{lcccccccccc}
\toprule
\multirow{2}{*}{\textbf{Models}} 
& \multicolumn{5}{c}{\textbf{QM9} ($|\mathcal{V}| \in [1,9]$, 134K mols.)} 
& \multicolumn{5}{c}{\textbf{ZINC250k} ($|\mathcal{V}| \in [6,38]$, 250K mols.)} \\
\cmidrule(lr){2-6} \cmidrule(lr){7-11}
& Val. $\uparrow$ & Uni. $\uparrow$ & Nov. $\uparrow$ & FCD $\downarrow$ & Time (s) $\downarrow$
& Val. $\uparrow$ & Uni. $\uparrow$ & Nov. $\uparrow$ & FCD $\downarrow$ & Time (s) $\downarrow$ \\
\midrule
\rowcolor{gray!10} 
\multicolumn{11}{c}{\textit{Diffusion Models}}\\
EDP-GNN   & 47.52 & 99.25 & 86.58 & 2.68 & 1.7$e^3$ & 82.97 & \textbf{99.79} & 100 & 16.74 & 1.6$e^3$ \\
GDSS      & 95.72 & 98.46 & 86.27 & 2.9 & 43.4 & \textbf{97.01} & 99.64 & 100 & \textbf{14.66} & 3.4$e^2$ \\
DiGress   & \textbf{99.0} & 96.66 & 33.4 & \textbf{0.36} & 34.1 & 91.02 & 81.23 & 100 & 23.06 & 2.0$e^2$ \\
\midrule
\rowcolor{gray!10} 
\multicolumn{11}{c}{\textit{Hybrid Models}}\\
GraphAF   & 74.43 & 88.64 & 86.59 & 5.27 & 1.1$e^3$ & 68.47 & 98.64 & 100 & 16.02 & 9.6$e^2$ \\
GraphDF   & 93.88 & 98.58 & 98.54 & 10.93 & 1.9$e^4$ & 90.61 & 99.63 & 100 & 33.55 & 9.3$e^3$ \\
\midrule
\rowcolor{gray!10} 
\multicolumn{11}{c}{\textit{Varational Autoencoders}}\\
GramVAE  & 20.00 & 19.70 & \textbf{100} & -- & --  & 30.10 & 27.30 & 100 & -- & -- \\
GraphVAE  & 45.80 & 30.50 & 66.10 & -- & -- & 44.00 & -- & -- & -- & -- \\
\midrule
\textbf{Ours}  & 89.43 & \textbf{99.31} & 83.14 & 5.77 & \textbf{3.84} & 84.22 & 87.88 & 100 & 19.70 & \textbf{19.7} \\
\textbf{$\pm$ std.}  
& $\pm$0.87 & $\pm$0.38 & $\pm$1.72& $\pm$0.49 & $\pm$1.40  
& $\pm$1.32 & $\pm$0.53 & $\pm$0.00 & $\pm$1.06 & $\pm$0.58 \\
\bottomrule
\end{tabular}
\end{table*}

\section{Future Work and Limitations}
This work primarily focused on introducing a novel paradigm for autoregressive graph generation, adopting a \textit{next-scale} prediction approach instead of traditional \textit{next-node} and \textit{next-edge} methods. As a result, a promising avenue for future research is to explore how autoregressive properties from the large language model (LLM) literature translate to graph generation. Specifically, efforts could be made to investigate the generalization capabilities of autoregressive models in zero-shot and few-shot downstream tasks, which have been shown to be limitations of state-of-the-art diffusion-based graph models \cite{vignac2023digress}. Furthermore, future work could explore fine-tuning strategies as a solution to tasks with little data and resources.

Although \texttt{MAG} demonstrated substantial speed improvements, experiments on larger graphs revealed qualitative limitations—likely stemming from our use of a fixed set of scales. Incorporating inductive biases or domain-specific signals into scale selection could enhance both performance and scalability on more complex structures. Moreover, autoregressive models are known for their inherent scalability, suggesting that \texttt{MAG}'s performance on larger graphs could further improve through optimized scale selection. A promising direction is to train a unified decoder that maps any intermediate latent scale back to a graph, much like \citeauthor{jiao2025flexvar}’s multi-resolution VAE. This would not only allow interpretation and visualization of the coarse-to-fine hierarchy and how motifs emerge, but also drive a data-driven choice for scale counts and resolutions, alleviating our current dependence on manually tuned scale hyperparameters.

Finally, this work opens a new direction for hybrid autoregressive-diffusion methods. While current hybrid approaches primarily rely on block-wise diffusion with sequential learning between blocks, the proposed framework enables reconsidering these methods by applying scale-wise diffusion, i.e., diffusion between scales. This shift could eliminate the need for a quantized autoencoder, as diffusion offers strong capabilities in continuous space, thereby removing the quantization bottleneck. However, this would reintroduce computational overheads which should be studied as a trade-off against existing diffusion-based methods.

\section{Conclusion}
This work introduced a novel diffusion-free multi-scale autoregressive model for graph generation which \textbf{(1)} theoretically addresses limitations of traditional autoregressive models due to inherent graph properties and \textbf{(2)} reduces the gap between diffusion-free AR models and diffusion-based methods, with up to three orders of magnitude lower computational costs. Experiments over both generic and molecular graph generations demonstrated the ability of \texttt{MAG} to generate high-quality samples that achieve promising performance and can adjust to domain-specific constraints such as chemical validity. We hope that our findings can aid the design of new approaches, in which the multi-scale paradigm may be applied to other graph generative methods.







\section*{Impact Statement}

Fast graph generation can accelerate scientific discovery by enabling rapid simulation of complex networks in fields such as epidemiology, ecology, and social science, lowering barriers for smaller research teams. It supports more resilient infrastructure and public‐health planning by allowing practitioners to model extreme scenarios at unprecedented speed. By democratizing access to advanced network modeling, these tools empower communities worldwide to tackle interconnected societal challenges.

\section*{Acknowledgment}
We thank Martinkus Karolis for his detailed feedback on architectural design, his comments on final drafts, and the provision of key literature references that strengthened our work. We also extend our gratitude to Petar Veli\v{c}kovi\'{c} for his early feedback on the foundations of this work and his continued support throughout.


\bibliography{example_paper}
\bibliographystyle{icml2025}


\appendix
\onecolumn
\section{Details for Hyperparameters} \label{hyperparameters}

\begin{table}[H]
\centering
\caption{Hyperparameters used for \texttt{MAG} during training and inference.}
\label{tab:hyperparameters}
\begin{tabular}{@{}l c@{}}
\toprule
\textbf{Hyperparameter}                        & \textbf{Value}                          \\ \midrule
\multicolumn{2}{c}{\emph{Multi‐Scale VQ‐VAE Tokenizer}}       \\ 
Encoder MPNN layers ($L_e$)              & 4                                       \\
Decoder GCN layers ($L_d$)               & 4                                       \\
Hidden dimension                   & 32                                         \\
Latent dimension ($C$)                   & 16                                     \\
Codebook size ($V$)                      & 1024                                     \\ 
Commitment cost                   & 0.25                                         \\
Gamma                           & 0.1                                         \\ \midrule
\multicolumn{2}{c}{\emph{Multi-scale Transformer}}              \\
Transformer blocks                       & 8                                      \\
Hidden size                              & 256                                     \\
Attention heads                          & 8                                       \\
Level embedding dim                      & 256                                     \\
Layer dropout                            & 0.1                                     \\
Conditional dropout                      & 0.1                                     \\
Token dropout                            & 0.05                                     \\ \midrule
\multicolumn{2}{c}{\emph{Optimization \& Sampling}}           \\
Optimizer                                & Adam                                    \\
Learning rate                            & $3\times10^{-5}$                        \\
Weight decay                             & $1\times10^{-2}$                        \\
Betas                                    & (0.9, 0.99)
                \\
Batch size                               & 12                                      \\
Training epochs                          & 100                                     \\
Top-$k$ sampling ($k$)                   & 50                                      \\
Top-$p$ sampling ($p$)         & 0.95                                    \\ \bottomrule
\end{tabular}
\end{table}

\section{Complexity Proof}
\label{complexity proof}

This proof entails the improvement in generation complexity from node-by-node to scale-wise AR models.

\begin{lemma}[AR Generation for Nodes]
For a standard self-attention transformer, the time complexity of autoregressive (AR) graph generation is $O(N^3)$, where $N$ is the total number of nodes.
\end{lemma}

\begin{proof}
In AR generation, nodes are generated sequentially. At step $i$ ($1 \leq i \leq N$), the model computes attention over all $i-1$ previously generated nodes. The complexity for step $i$ is $O(i^2)$. Summing over all steps:
\[
\sum_{i=1}^{N} i^2 = \frac{N(N+1)(2N+1)}{6} \sim O(N^3).
\]
\end{proof}
\begin{lemma}[Scale-wise Generation for Nodes]
For a standard self-attention transformer with constant $a > 1$, the time complexity of variable autoregressive (VAR) graph generation is $O(N^2)$, where $N$ is the total number of nodes, and $K = \log_a N + 1$ scales are used.
\end{lemma}
\begin{proof}
Define the resolution sequence $\{n_k\}$, where $n_k = a^{k-1}$ nodes are added at scale $k$, and $n_K = N$. The total nodes up to scale $k$ is:
\[
S_k = \sum_{i=1}^k n_i = \frac{a^k - 1}{a - 1}.
\]
At each scale $k$, generating $n_k$ new nodes requires computing attention over all $S_{k-1}$ existing nodes. The complexity is:
\[
n_k \cdot S_{k-1} \approx a^{k-1} \cdot a^{k-1} = a^{2(k-1)}.
\]
Summing over all $K$ scales:
\[
\sum_{k=1}^K a^{2(k-1)} = \frac{a^{2K} - 1}{a^2 - 1} \sim O(a^{2K}) = O(N^2),
\]
since $a^{K-1} = N$ and thus $a^{2K} = a^2 N^2$.
\end{proof}

\section{Details for Evaluation Metrics} \label{app:metrics}

To evaluate the performance of \texttt{MAG} for generic graphs, we measure the degree, clustering, and orbit of the generated graphs. Degree measures the number of connections of each node, revealing hubs or sparsely connected regions. Clustering shows how likely a node’s neighbors are to be connected, indicating local group structures or tightly-knit communities. Orbit counts specific small patterns (subgraphs) within the graph, capturing repeating structural patterns. These metrics report whether the structure of the graphs in the dataset is preserved, indicating more meaningful and realistic graphs.

We assess molecular graph generation by examining validity, uniqueness, novelty, and the Frechet ChemNet Distance (FCD) metric. Validity reflects the proportion of generated structures that adhere to fundamental chemical rules—such as correct valency—without any post-hoc correction. Uniqueness gauges the degree to which those valid molecules are distinct from one another, while novelty captures the extent to which they lie outside the original training corpus. Finally, FCD quantifies how closely the distribution of generated compounds matches that of real chemicals by comparing the mean and covariance of their activations in ChemNet’s penultimate layer.


\end{document}